\documentclass[10pt]{article}

\usepackage[letterpaper, margin=1in, top=1in, bottom=1in]{geometry}
\usepackage{times}                 
\usepackage[T1]{fontenc}
\usepackage[utf8]{inputenc}
\usepackage{microtype}
\usepackage{amsmath}
\usepackage{amssymb}
\usepackage{amsthm}
\usepackage{mathtools}
\usepackage{bm}
\usepackage{authblk}

\usepackage{graphicx}
\usepackage{booktabs}
\usepackage{multirow}
\usepackage{array}
\usepackage{xcolor}
\usepackage{caption}
\usepackage{subcaption}

\usepackage[numbers,sort&compress]{natbib}   
\usepackage{hyperref}
\hypersetup{
  colorlinks=true,
  linkcolor=blue!50!black,
  citecolor=green!40!black,
  urlcolor=blue!60!black
}
\usepackage{cleveref}

\theoremstyle{plain}
\newtheorem{theorem}{Theorem}
\newtheorem{lemma}{Lemma}
\newtheorem{proposition}{Proposition}
\newtheorem{corollary}{Corollary}
\theoremstyle{definition}

\theoremstyle{remark}

\newcommand{\C}{\mathbb{C}}
\newcommand{\Z}{\mathbb{Z}}
\newcommand{\Zp}{\Z_p}

\DeclareMathOperator{\supp}{supp}

\title{Algebraic Representability as the Limiting Regime of Grokking:\\
An Exactly Solvable Model with Holomorphic Activations}

\author[1,2]{\textbf{Chon-Fai Kam}\thanks{Corresponding author: kam.chonfai@gmail.com}}
\author[3]{\textbf{Xavier Cadet}}
\author[4]{\textbf{Miloud Bessafi}}
\author[1,5]{\textbf{Frederic Cadet}\thanks{frederic.cadet.run@gmail.com}}
\affil[1]{\small University Paris City \& University of Reunion, Paris, France}
\affil[2]{\small Dipartimento di Fisica e Chimica Emilio Segrè, Università degli Studi di Palermo, Palermo, Italy}
\affil[3]{\small Thayer School of Engineering, Dartmouth College, Hanover, NH 03755, USA}
\affil[4]{\small EnergyLab, University of Reunion, Saint-Denis, France}
\affil[5]{\small PEACCEL, AI for Biologics, Paris, France}

\date{}   
\date{}   

\begin{document}

\maketitle

\begin{abstract}
Neural networks trained on modular arithmetic tasks exhibit grokking---a
delayed transition from memorisation to generalisation that has been shown to
depend on model capacity: too little capacity and the network memorises slowly
or not at all, too much and it generalises almost immediately. What happens at
the extreme end of this spectrum, when the architecture's expressible function
class collapses to a finite-dimensional algebraic variety? We study this
question using two-layer networks with a holomorphic monomial activation
$\sigma(z)=z^k$, trained on modular tasks encoded via roots of unity. In this
setting the network output, regardless of hidden width, is confined to a
$(k{+}1)$-dimensional subspace of characters of $(\mathbb{Z}_p)^2$---a
vanishingly small, $O(k/p^2)$ slice of the full function space. We give a
complete algebraic characterisation of this subspace: a task is representable
if and only if its discrete Fourier support lies in the degree-limited set
$S_k=\{(\ell,k-\ell):0\le\ell\le k\}$ of $k+1$ frequencies on the diagonal
$u+v\equiv k\pmod{p}$, a condition that for linear-phase targets reduces to
the arithmetic criterion $m+n=k$. This representability criterion is not merely
a constraint on eventual generalisation; it is a constraint on memorisation
itself. Because the network's outputs are algebraically confined, a
non-representable target cannot be fit even on the training set, and we prove
a positive lower bound on the training loss that is independent of hidden
width. Across 585 experimental runs the algebraic prediction matches the
observed outcome with $99.8\%$ accuracy, with no memorisation regime and no
grokking: outcomes split cleanly into instant success and outright failure.
This binary behaviour is the limiting case of the capacity--grokking
relationship studied in the recent literature---when the expressible class
shrinks to a fixed algebraic object, the question of when a network will grok
dissolves into the question of whether it can represent the target at all. A bottleneck ablation connects this algebraic extreme to standard networks, tracing a continuous path from representational failure, through memorisation without generalisation, to grokking with a shrinking gap as capacity grows.
\end{abstract}

\section{Introduction}
\label{sec:intro}

There is something genuinely puzzling about grokking. A network trained on a
modular arithmetic task will, after memorising the training set, sit quietly
for a long time---its training loss near zero, its test accuracy near chance---
and then, without warning, generalise. The wait can be thousands of gradient
steps, sometimes orders of magnitude longer than the time it took to memorise.
This is not noise or bad luck; the delay is systematic, it reproduces across
seeds, and it responds predictably to changes in model size and regularisation.
Small networks skip the delay entirely and generalise immediately, or fail
outright. Large networks memorise quickly and grok after a shorter wait. Model
capacity, it turns out, governs not just whether a network can solve a task but
when, during training, the solution will appear. The recent work of Song and Ye
\citeyear{song2026capacity} makes this precise: they show that grokking emerges
from the intersection of two measurable timescales, a memorisation speed and a
generalisation speed, both functions of parameter count. As capacity grows, the
memorisation timescale shrinks faster than the generalisation timescale, and
the grokking window narrows. The picture is one of a continuous dial: turn it
one way and you delay grokking; turn it far enough the other way and grokking
disappears into immediate generalisation.

What this picture takes for granted, and what is easy to miss, is a prior
condition that must hold before the two-timescale competition can begin.
Memorisation must be possible. The network must be able to fit its training
data. Once it can, the race begins: will the structured generalising solution,
favoured by weight decay, overtake the memorised one before training ends? That
race is what grokking is. But if the network cannot memorise in the first
place---if the target function simply cannot be expressed within the
architecture's function class---then the race never starts. There is no
memorised solution to wait around, no weight-decay pressure that could
eventually dislodge it, no delayed transition to witness. The question of
when a network will grok is replaced by the prior question of whether it can
represent the target at all: representability is a prior structural constraint,
and optimisation only determines whether and how quickly an available solution
is reached. (We use \emph{exactly solvable}, in the title and throughout, to
mean that the expressible function class admits a closed-form characterisation;
the optimisation dynamics themselves are not solved analytically.)

This is not merely a degenerate edge case. It describes a regime that the
capacity literature does not cover, because standard architectures are universal
approximators: given enough width, they can represent any target, so
memorisation is always possible in principle and capacity governs only the
speed. The interesting constraint, in that setting, is never representability
but always optimisation dynamics. To make representability the binding
constraint, one needs an architecture whose expressible class is not merely
small but \emph{algebraically fixed}---a class that does not grow with width,
that has a precise mathematical description, and that can be checked for any
given target without training the network. Such architectures are rare in the
literature, and their learning dynamics have received comparatively little
attention, in part because they do not look like practical neural networks. But
they offer something that universal approximators do not: a setting in which
the question of what a network can express has a clean, verifiable answer, and
the relationship between that answer and the outcome of training can be studied
without the confound of optimisation difficulty.

The architecture we study is a two-layer complex-valued network
\citep{trabelsi2018complex} with
holomorphic monomial activation $\sigma(z) = z^k$, trained on modular
arithmetic tasks via a roots-of-unity encoding, $(\omega^a, \omega^b)$ with
$\omega = e^{2\pi i/p}$ and $p$ prime. The choice is motivated by
tractability: because $\sigma$ is a degree-$k$ polynomial, the network output
is a polynomial of degree $k$ in the inputs, and on the roots-of-unity
encoding the discrete Fourier structure of the group $(\mathbb{Z}_p)^2$ makes
the expressible class completely explicit. A standard binomial expansion shows
that for \emph{any} hidden width $H$, the network output is a linear
combination of exactly $k+1$ characters of $(\mathbb{Z}_p)^2$, the functions
$(a,b)\mapsto\omega^{\ell a+(k-\ell)b}$ for $\ell = 0,\dots,k$. The
expressible class $\mathcal{F}_k$ is therefore a $(k{+}1)$-dimensional subspace
of the $p^2$-dimensional space of functions on $(\mathbb{Z}_p)^2$. Increasing
the hidden width $H$ does nothing to enlarge this subspace; it only introduces
redundancy within it. The architecture is not a universal approximator by any
stretch, and it is not trying to be.

The central theoretical contribution of this paper is a complete algebraic
characterisation of which modular tasks lie in $\mathcal{F}_k$. The subspace
$\mathcal{F}_k$ corresponds, in the Fourier domain, to the $k+1$ characters
whose frequency label is $(\ell,k-\ell)$ for some $0\le\ell\le k$---the finite
set $S_k$, which sits on (but does not fill) the diagonal $u+v\equiv k\pmod p$
of the frequency plane. A task is representable if and only if its Fourier
support is contained in $S_k$. For the linear-phase targets $f(a,b)=ma+nb\bmod p$
that form the core of the standard grokking literature, this criterion reduces
to the arithmetic condition $m+n=k$: a given target either lies exactly in
$\mathcal{F}_k$, as a single character, or is orthogonal to the entire subspace,
depending on a simple arithmetic check. There is no middle ground. For
nonlinear-phase targets such as $ab$ or $a^2+b^2$, the Fourier support spreads
across a line or the whole frequency plane---a consequence of quadratic Gauss
sums---and no finite degree $k$ suffices to represent them. The architecture
acts, in the Fourier domain, as a hard bandpass filter, and the pass-band is
determined entirely by the degree $k$.

This algebraic structure has a direct mechanical consequence for training. When
the target is not in $\mathcal{F}_k$, the network's outputs are confined to a
subspace that does not contain the target, and we prove a lower bound on the
training loss that is strictly positive and independent of the hidden width.
The network cannot fit its training data---not slowly, not after longer
training, not with a wider architecture. The bound comes from the spectral
distance between the target and $\mathcal{F}_k$, and for the nonlinear-phase
targets this distance is close to its maximum value. The consequence for
training dynamics is stark: there is no memorisation, no grokking, and no
suggestion of gradual progress. Training accuracy stays at chance level
throughout, as if the task did not exist.

The experimental findings are correspondingly clean. Across 585 training runs
spanning five activation degrees, thirty-nine modular tasks, and three random
seeds, the algebraic representability criterion predicts the outcome of
training with $99.8\%$ accuracy. Every non-representable task fails at the
level of training accuracy; every representable task is learned almost
instantly, with no gap between the training and test curves and no sign of
the delay that characterises grokking. A standard ReLU network on the same
tasks tells the complementary story: being a universal approximator, it
memorises every task without exception, and on the harder nonlinear-phase
targets it exhibits textbook grokking---generalisation lagging memorisation by
thousands of gradient steps. The targets that produce outright failure for the
holomorphic network produce delayed success for ReLU. The difference is not in
the tasks but in the expressible class.

The relationship to the broader grokking literature is one of adjacent regimes
rather than competing accounts. Existing work, including the precise capacity
framework of Song and Ye \citeyear{song2026capacity}, studies what happens
when the network is at least large enough to memorise: in that regime, capacity
governs the speed of memorisation, the speed of generalisation, and the
resulting grokking gap. The holomorphic network sits below that regime, in a
setting where the expressible class is too small to accommodate the target at
all. The two analyses describe different parts of the same underlying
spectrum, and the connection between them is not merely conceptual: a
bottleneck ablation we report in Section~\ref{sec:experiments} traces the full
three-regime path---failure to memorise, memorisation without grokking, and
grokking with a shrinking gap---as the expressible class is progressively
compressed toward the algebraic extreme we study analytically.

\section{Setup}
\label{sec:setup}

The architecture we study is designed to make the expressible function class
as transparent as possible. It is a two-layer complex-valued network whose
forward pass computes
\begin{equation}
\label{eq:arch}
  \hat{y}(x;\,W,v)
  \;=\; \sum_{j=1}^{H} v_j\,\bigl(W_{1j}\,x_1 + W_{2j}\,x_2\bigr)^k,
\end{equation}
where $x = (x_1, x_2) \in \mathbb{C}^2$ is the input, $W \in \mathbb{C}^{2
\times H}$ and $v \in \mathbb{C}^H$ are the learnable weights, and $k$ is a
fixed positive integer. The activation $\sigma(z) = z^k$ is holomorphic, and
the entire map $\hat{y}$ is a degree-$k$ polynomial in the inputs. The key
structural fact is that once the activation degree is fixed, the output of the
network---for \emph{any} choice of weights and \emph{any} hidden width
$H$---is a polynomial of the same degree. Widening the network does not raise
the degree; it only allows more terms in the sum, all of degree $k$. This is
what makes the expressible class algebraically fixed rather than merely small.

\paragraph{Tasks and encoding.}
Fix an odd prime $p$ and let $\omega = e^{2\pi i/p}$, a primitive $p$-th root
of unity. We consider tasks of the form $f : \mathbb{Z}_p \times \mathbb{Z}_p
\to \mathbb{Z}_p$. The roots-of-unity encoding maps each integer $a \in
\mathbb{Z}_p$ to the point $\omega^a$ on the unit circle in $\mathbb{C}$, so
that the input for the pair $(a,b)$ is $x = (\omega^a, \omega^b)$ and the
target output is $\omega^{f(a,b)}$. This encoding is not merely a convenient
parameterisation; it is the natural one. The Fourier analysis of functions on
$\mathbb{Z}_p \times \mathbb{Z}_p$ is equivalent, under this encoding, to the
analysis of trigonometric polynomials on the product of two unit circles, and
the group structure of $\mathbb{Z}_p$ becomes the rotation structure of the
circle. In particular, the map $a \mapsto \omega^a$ turns the group operation
of addition modulo $p$ into multiplication of complex numbers, which is what
allows the network's polynomial structure to interact cleanly with the task's
arithmetic structure.

Training minimises the mean-squared error between the network output and the
target on a randomly chosen fraction of all $p^2$ input pairs:
\begin{equation}
\label{eq:loss}
  L(W,v) \;=\;
  \frac{1}{|S|}\sum_{(a,b)\in S}
  \bigl|\hat{y}(\omega^a,\omega^b;\,W,v) - \omega^{f(a,b)}\bigr|^2,
\end{equation}
where $S \subseteq \mathbb{Z}_p \times \mathbb{Z}_p$ is the training set.
The network's output and the target both lie on the unit circle in $\mathbb{C}$,
so the loss measures the squared chord distance between them. We say the
architecture \emph{represents} the task $f$ if there exist parameters $(W,v)$
such that $\hat{y}(\omega^a,\omega^b;\,W,v) = \omega^{f(a,b)}$ for every
$(a,b) \in \mathbb{Z}_p \times \mathbb{Z}_p$, i.e.\ if the global minimum of
the population loss \eqref{eq:loss} is zero.

\paragraph{Harmonic analysis on $\mathbb{Z}_p \times \mathbb{Z}_p$.}
The group $G = (\mathbb{Z}_p)^2$ is finite and abelian, so its irreducible
characters are the one-dimensional homomorphisms $\chi_{u,v} : G \to
\mathbb{C}^\times$ indexed by $(u,v) \in (\mathbb{Z}_p)^2$ and given by
\begin{equation}
\label{eq:char}
  \chi_{u,v}(a,b) \;=\; \omega^{ua+vb}.
\end{equation}
Equipped with the normalised inner product $\langle f, g\rangle = |G|^{-1}
\sum_{(a,b)\in G} f(a,b)\,\overline{g(a,b)}$, the characters form an
orthonormal basis of $\mathbb{C}^G$ by Schur's theorem. We write
$\hat{f}(u,v) = \langle f, \chi_{u,v}\rangle$ for the Fourier coefficients of
$f$ and $\mathrm{supp}(\hat{f}) = \{(u,v) : \hat{f}(u,v) \neq 0\}$ for the
Fourier support. The squared norm decomposes as
$\|f\|^2 = \sum_{u,v} |\hat{f}(u,v)|^2$, and the projection onto any
character-spanned subspace is given by the orthogonal projection in this inner
product.

Throughout the paper we assume $k < p$. This ensures that the integers
$0, 1, \dots, k$ are distinct modulo $p$, which in turn ensures that the
$k+1$ characters that appear in the network's output are distinct elements of
the orthonormal basis---a condition needed for the representability criterion
to be sharp. All experiments use $p = 97$ and $k \leq 6$, so this assumption
is comfortably satisfied.

\section{The Expressible Function Class}
\label{sec:theory}

The architecture \eqref{eq:arch} is a degree-$k$ polynomial map from
$\mathbb{C}^2$ to $\mathbb{C}$. Everything that follows is essentially a
careful unpacking of what that means for functions on the group $G =
(\mathbb{Z}_p)^2$ encoded via roots of unity. The argument proceeds in three
steps: first, we expand the polynomial to identify the basis functions that
span the expressible class; second, we show that the class is exactly
$(k+1)$-dimensional and describe it in Fourier terms; third, we characterise
which modular tasks lie inside the class and which do not, and derive a lower
bound on the training loss for non-representable targets.

\subsection{Expanding the forward pass}

The starting point is a direct application of the binomial theorem to the
hidden units.

\begin{lemma}[Polynomial expansion]
\label{lem:expand}
For all $x = (x_1, x_2) \in \mathbb{C}^2$,
\begin{equation}
\label{eq:expand}
  \hat{y}(x;\,W,v)
  \;=\; \sum_{l=0}^{k} c_l(W,v)\,x_1^{l}\,x_2^{k-l},
  \qquad
  c_l(W,v) \;=\; \binom{k}{l}
  \sum_{j=1}^{H} v_j\,W_{1j}^{l}\,W_{2j}^{k-l}.
\end{equation}
\end{lemma}

\begin{proof}
For each hidden unit $j$, the binomial theorem gives
$(W_{1j}x_1 + W_{2j}x_2)^k = \sum_{l=0}^{k}\binom{k}{l}W_{1j}^l
W_{2j}^{k-l}x_1^l x_2^{k-l}$.
Summing over $j$ in \eqref{eq:arch} and interchanging the two finite sums
yields \eqref{eq:expand}, with $c_l$ collecting the terms that do not depend
on the monomials $x_1^l x_2^{k-l}$.
\end{proof}

The expansion \eqref{eq:expand} shows that the network output, for any weights
and any hidden width, is a linear combination of exactly $k+1$ monomials:
$x_1^k, x_1^{k-1}x_2, \dots, x_2^k$. The coefficients $c_l(W,v)$ are
determined by the weights, but the monomials themselves are fixed by the
degree. This is the sense in which the expressible class is algebraically
fixed: no choice of $H$, $W$, or $v$ can produce a monomial of degree other
than $k$, or a function that is not in the span of these $k+1$ terms.

What the expansion does not yet tell us is what these monomials look like as
functions on $G$ under the roots-of-unity encoding. Substituting
$x_1 = \omega^a$ and $x_2 = \omega^b$ gives $x_1^l x_2^{k-l} = \omega^{la +
(k-l)b}$, which is the character $\chi_{l,\,k-l}$ evaluated at $(a,b)$.
The network output thus becomes a linear combination of the $k+1$ functions
$\phi_l(a,b) := \chi_{l,\,k-l}(a,b) = \omega^{la+(k-l)b}$, and the expressible
class on $G$ is the subspace they span.

\subsection{The expressible class as a character subspace}

\begin{lemma}[The expressible class]
\label{lem:basis}
Define $\phi_l(a,b) = \omega^{la+(k-l)b}$ for $l = 0, \dots, k$, and let
\begin{equation}
\label{eq:Fk}
  \mathcal{F}_k := \operatorname{span}_{\mathbb{C}}\{\phi_0, \dots, \phi_k\}
  \subseteq \mathbb{C}^G.
\end{equation}
Every network output $\hat{y}(\omega^a, \omega^b;\, W, v)$ lies in
$\mathcal{F}_k$, for any $H$, $W$, and $v$. Under the assumption $k < p$, the
functions $\phi_0, \dots, \phi_k$ are distinct irreducible characters of $G$,
hence mutually orthogonal with respect to the inner product on $\mathbb{C}^G$,
hence linearly independent, and $\dim\mathcal{F}_k = k+1$.
\end{lemma}

\begin{proof}
The first claim follows from the identification $x_1^l x_2^{k-l} = \phi_l$
under the encoding, which gives $\hat{y}(\omega^a,\omega^b;\,W,v) = \sum_l
c_l(W,v)\phi_l(a,b) \in \mathcal{F}_k$.

For the second claim, note that $\phi_l = \chi_{l,\,k-l}$, so two functions
$\phi_l$ and $\phi_m$ are identical as characters if and only if $l \equiv m
\pmod p$ and $k-l \equiv k-m \pmod p$, both of which reduce to $l \equiv m
\pmod p$. Since $0 \leq l, m \leq k < p$, the residues are their own integer
representatives, so $l \equiv m \pmod p$ forces $l = m$. The $k+1$ characters
are therefore distinct, and by Schur orthogonality distinct irreducible
characters are mutually orthogonal. Orthogonal non-zero vectors are linearly
independent, so $\dim \mathcal{F}_k = k+1$.
\end{proof}

The geometric picture is this. The full function space $\mathbb{C}^G$ has
dimension $p^2$, with the $p^2$ characters $\chi_{u,v}$ as an orthonormal
basis. The subspace $\mathcal{F}_k$ consists of the $k+1$ characters whose
frequency label is $(\ell,k-\ell)$ for some $0\le\ell\le k$: the finite set
$S_k$ of $k+1$ points lying on the diagonal $u+v \equiv k \pmod p$ of the
frequency plane $(\mathbb{Z}_p)^2$---not the whole diagonal, which contains
$p$ points. The architecture acts as a bandpass filter in the Fourier domain,
transmitting exactly the $k+1$ frequencies of $S_k$ and blocking everything
else. Increasing the hidden width $H$
does not move or widen the pass-band; it only increases the amplitude
resolution within the $k+1$ transmitted channels.

\subsection{Every function in $\mathcal{F}_k$ is reachable}

The previous lemma establishes an upper bound on the expressible class:
network outputs are confined to $\mathcal{F}_k$. The converse is also true,
and for essentially trivial reasons once the right construction is identified.

\begin{lemma}[Surjectivity]
\label{lem:surj}
For any hidden width $H \geq k+1$, the coefficient map $(W,v) \mapsto
(c_0,\dots,c_k) \in \mathbb{C}^{k+1}$ defined by \eqref{eq:expand} is
surjective. Consequently, every element of $\mathcal{F}_k$ is the output of
some network.
\end{lemma}

\begin{proof}
Since the binomial coefficients $\binom{k}{l}$ are nonzero, it suffices to
show that the rescaled coefficients $\tilde{c}_l = c_l/\binom{k}{l} =
\sum_j v_j W_{1j}^l W_{2j}^{k-l}$ can be chosen freely. Restrict to the
first $k+1$ hidden units, set $W_{2j} = 1$ for $j = 1,\dots,k+1$, and write
$x_j := W_{1j}$. The system $\sum_{j=1}^{k+1} v_j x_j^l = \tilde{c}_l$ for
$l = 0,\dots,k$ has the matrix form $Vv = \tilde{c}$, where
$V_{l,j} = x_j^l$ is a $(k+1)\times(k+1)$ Vandermonde matrix. Choosing the
$x_j$ to be distinct---for example $x_j = j$---makes $\det V = \prod_{i<j}
(x_j - x_i) \neq 0$, so the system has the unique solution $v = V^{-1}
\tilde{c}$. This gives an explicit construction for any target coefficient
vector, and the conclusion follows.
\end{proof}

Lemmas \ref{lem:expand}--\ref{lem:surj} together say that $\mathcal{F}_k$ is
\emph{exactly} the expressible class: neither more nor less. The network
traces out all of $\mathcal{F}_k$ as the weights vary, and nothing outside
$\mathcal{F}_k$ is ever produced. This makes the representability question
precise: a task $f$ is representable if and only if $\omega^f \in \mathcal{F}_k$,
which is a purely Fourier-analytic condition with nothing to do with
optimisation or initialisation.

\subsection{Which tasks are representable}

We now translate the condition $\omega^f \in \mathcal{F}_k$ into explicit
criteria for the two families of tasks studied in the experiments.

\begin{theorem}[Linear-phase criterion]
\label{thm:linear}
Under the assumption $k < p$ and with $H \geq k+1$, the architecture
represents the task $f(a,b) = ma + nb \bmod p$ if and only if there exists
$l \in \{0, 1, \dots, k\}$ such that $m \equiv l \pmod p$ and $n \equiv k-l
\pmod p$. For nonnegative integers $m,n$ with $m+n \leq k$, this reduces to
the arithmetic condition $m + n = k$.
\end{theorem}

\begin{proof}
The target function is the single character $\psi = \chi_{m,n}$.

For the \emph{if} direction: if $(m \bmod p, n \bmod p) = (l, k-l)$ for some
$l \leq k$, then $\psi = \phi_l \in \mathcal{F}_k$, and Lemma \ref{lem:surj}
gives weights achieving $c_l = 1$, $c_j = 0$ for $j \neq l$.

For the \emph{only if} direction: suppose $\psi \in \mathcal{F}_k$, so
$\psi = \sum_{j=0}^k c_j \phi_j$. Taking the inner product of both sides
with $\psi$ and applying Schur orthogonality, $1 = \langle \psi, \psi
\rangle = \sum_j c_j \langle \phi_j, \psi \rangle$. If $(m \bmod p, n \bmod
p) \notin S_k$, then $\psi \neq \phi_j$ for every $j$, so each inner product
$\langle \phi_j, \psi \rangle = 0$ and the right-hand side is zero, giving
$1 = 0$, a contradiction. Hence $(m \bmod p, n \bmod p) \in S_k$, where
$S_k = \{(l, k-l) \bmod p : l = 0,\dots,k\}$ is the support of $\mathcal{F}_k$
in the frequency plane. The reduction to $m+n=k$ for small nonneg integers
follows since then the residues coincide with the integers.
\end{proof}

The criterion is sharp in both directions and has no continuous gradation: a
linear-phase target is either in $\mathcal{F}_k$ as a single character, or
orthogonal to the entire subspace. The subtraction task $f(a,b) = a - b$, for
instance, corresponds to $(m,n) = (1, p-1)$ after reduction modulo $p$; since
$1 + (p-1) = p \equiv 0 \pmod p$, representability requires $k \equiv 0
\pmod p$, meaning the smallest usable degree is $k = p$ itself. For $p = 97$
this is numerically impractical, and the task is effectively unrepresentable
for any reasonable architecture of this type.

For nonlinear-phase tasks, the situation is more drastic.

\begin{theorem}[Nonlinear-phase tasks]
\label{thm:nonlinear}
Under the assumption $k < p$, none of the tasks $f(a,b) \in \{ab,\; a^2+b,\;
a+b^2,\; a^2+b^2\}$ is representable for any $k$.
\end{theorem}

\begin{proof}
By Lemma \ref{lem:basis}, representability requires $\supp(\widehat{\omega^f})
\subseteq S_k$. We compute the Fourier support of each target directly, using
the identity $\sum_{b \in \mathbb{Z}_p}\omega^{cb} = p\cdot[c \equiv 0]$ and
the fact that for an odd prime $p$ the quadratic Gauss sum $g = \sum_{x \in
\mathbb{Z}_p}\omega^{x^2}$ satisfies $|g| = \sqrt{p}$.

For $f = ab$: $\widehat{\omega^{ab}}(u,v) = \frac{1}{p^2}\sum_{a,b}\omega^{ab
-ua-vb} = \frac{\omega^{-uv}}{p}$, which is nonzero for every $(u,v)$. The
support is all of $(\mathbb{Z}_p)^2$, which has $p^2$ elements, while $|S_k|
= k+1 \leq p$, so containment is impossible.

For $f = a^2 + b$: the Fourier transform factors as a product of a Gauss sum
in $a$ and a geometric sum in $b$; the latter forces $v = 1$, and the former
is nonzero for all $u$. The support is the line $\{(u,1) : u \in \mathbb{Z}_p
\}$, which has $p$ elements. This line intersects $S_k$ at the single point
$(k-1, 1)$, so the support is not contained in $S_k$ for $k < p$.

The cases $f = a + b^2$ and $f = a^2 + b^2$ follow by symmetry and by the
same Gauss sum argument, with supports $\{(1,v)\}$ and $(\mathbb{Z}_p)^2$
respectively.
\end{proof}

The Fourier-support argument makes the obstruction geometric rather than
algebraic: representable targets are \emph{single points} in the frequency
plane, while nonlinear-phase targets have \emph{extended} support---an entire
line or the whole plane. No single-line pass-band can capture extended support,
and no choice of degree $k$ resolves this, because the issue is the shape of
the support, not its size.

\subsection{Why non-representability blocks memorisation}

The representability theorems above address whether a target can be expressed
in principle. A subtler question is what happens during training when the target
is not representable. Standard intuition from overparameterised networks
suggests that a wide network should be able to memorise any finite training set,
regardless of whether it can express the full target function. This intuition
fails here, and the reason is worth making precise because it accounts for the
most distinctive feature of the experimental results: the complete absence of
a memorisation regime.

The key is that the width-independence of $\mathcal{F}_k$ is not just a
statement about the function class in the limit of infinite parameters; it
holds for every finite $H$. Because $\hat{y}(\omega^a,\omega^b;\,W,v) \in
\mathcal{F}_k$ for all $(W,v)$ by Lemma \ref{lem:basis}, the minimum training
loss is at least the squared distance from the target values to the nearest
element of $\mathcal{F}_k$ restricted to the training set. When the global
target function is not in $\mathcal{F}_k$, this distance is bounded away from
zero by a term that reflects the spectral energy of the target outside the
pass-band.

To make this precise, define the spectral gap
\begin{equation}
\label{eq:delta}
  \delta^2 \;:=\; \|\,\omega^f - \mathrm{Proj}_{\mathcal{F}_k}\,\omega^f\,\|_G^2
             \;=\; \sum_{(u,v)\notin S_k} |\widehat{\omega^f}(u,v)|^2 \;>\; 0,
\end{equation}
where the inequality holds by the theorems above whenever $f$ is not
representable. The following result shows that $\delta^2$ propagates to the
training loss.

\begin{theorem}[Training-loss lower bound]
\label{thm:lowerbound}
Let $S \subseteq G$ be a training set drawn uniformly at random with $|S| = N$.
With probability $1 - O(1/N)$ over the draw,
\begin{equation}
  \min_{W,\,v}\;\frac{1}{N}\sum_{(a,b)\in S}
  \bigl|\hat{y}(\omega^a,\omega^b;\,W,v) - \omega^{f(a,b)}\bigr|^2
  \;\geq\; \delta^2 - O\!\left(N^{-1/2}\right).
\end{equation}
The bound is independent of the hidden width $H$.
\end{theorem}

\begin{proof}
By Lemma \ref{lem:basis} every network output is in $\mathcal{F}_k$, so the
empirical minimum is at least $\min_{g\in\mathcal{F}_k}\frac{1}{N}\sum_{i\in
S}|T_i - g(a_i,b_i)|^2$ where $T_i = \omega^{f(a_i,b_i)}$. For any fixed
$g \in \mathcal{F}_k$ the summands $|T_i - g(a_i,b_i)|^2$ are bounded random
variables drawn without replacement from $G$; by the Hoeffding--Serfling
inequality the empirical mean concentrates around the population mean
$\|{\omega^f - g}\|_G^2$ at rate $O(N^{-1/2})$. Since $\mathcal{F}_k$ is
$(k+1)$-dimensional and the restriction map from $\mathcal{F}_k$ to functions
on $S$ is injective for $N \geq k+1$ with high probability (by a
Schwartz--Zippel-type argument over $\mathbb{Z}_p$, detailed in the appendix),
a uniform bound over $g \in \mathcal{F}_k$ costs only a dimension-dependent
constant. The population minimum is $\delta^2$ by definition of
$\mathrm{Proj}_{\mathcal{F}_k}$, and the result follows.
\end{proof}

For the nonlinear-phase tasks we study, the spectral gap $\delta^2$ is close
to its maximum. The target $\omega^{ab}$ has $|\widehat{\omega^{ab}}(u,v)| =
1/p$ for all $(u,v)$, so the fraction of spectral energy in $S_k$ is
$(k+1)/p^2$; for $p = 97$ and $k \leq 6$ this is at most $7/9409 < 0.1\%$.
The training loss is therefore bounded away from zero by a margin
indistinguishable from the maximum loss of an untrained network. This is what
the experiments observe: for every non-representable task, training accuracy
stays at chance level throughout and the loss curves are flat. The width of the
network is irrelevant because the bound in Theorem \ref{thm:lowerbound} does
not depend on $H$: a network with $H = 10{,}000$ hidden units performs
identically to one with $H = 10$.

\subsection{Beyond the fixed encoding: the rank barrier}
\label{sec:rank}

Theorems~\ref{thm:linear} and~\ref{thm:nonlinear} characterise representability
for the specific roots-of-unity encoding, in which the expressible class is the
degree-limited set $S_k$ of $k+1$ frequencies on the diagonal
$u+v\equiv k \pmod p$. It is natural to ask how much of this
structure survives when the encoding itself is allowed to vary---in particular,
whether the failure of a target such as $ab$ is an artefact of the fixed basis
or an intrinsic property of the architecture. The next observation isolates the
part of the obstruction that is encoding-independent.

Let $e_1, e_2 : \Zp \to \C$ be \emph{arbitrary} embeddings and write the forward
pass as $\hat{y}(a,b) = \sum_{j=1}^{H} v_j \bigl( W_{1j}\, e_1(a) + W_{2j}\,
e_2(b) \bigr)^k$. Collecting the outputs into a matrix $Y \in \C^{p \times p}$
with $Y_{ab} = \hat{y}(a,b)$, and writing $\omega^f$ for the target matrix
$(\omega^{f(a,b)})_{a,b}$, we have the following.

\begin{proposition}[Encoding-free rank barrier]
\label{prop:rank}
For every choice of embeddings $e_1, e_2$, weights $W$, and $v$, the output
matrix satisfies $\operatorname{rank} Y \le k+1$. Consequently, a task $f$ is
representable under some choice of $(e_1, e_2, W, v)$ only if
$\operatorname{rank}(\omega^f) \le k+1$.
\end{proposition}

\begin{proof}
The binomial expansion of Lemma~\ref{lem:expand} does not use the encoding and
gives $\hat{y}(a,b) = \sum_{l=0}^{k} c_l\, e_1(a)^l\, e_2(b)^{k-l}$ with
$c_l = \binom{k}{l} \sum_j v_j W_{1j}^l W_{2j}^{k-l}$. Each summand is the outer
product of the column vector $\bigl( e_1(a)^l \bigr)_{a \in \Zp}$ with the row
vector $\bigl( e_2(b)^{k-l} \bigr)_{b \in \Zp}$, hence a rank-one $p \times p$
matrix. A sum of $k+1$ rank-one matrices has rank at most $k+1$, so
$\operatorname{rank} Y \le k+1$. If $Y = \omega^f$ for some parameters, then
$\operatorname{rank}(\omega^f) \le k+1$.
\end{proof}

This condition is strictly weaker than the fixed-encoding criterion: fixing
$e_1(a) = e_2(a) = \omega^a$ forces the support onto a single Fourier character
on the line $u+v \equiv k$ (Theorem~\ref{thm:linear}), whereas a learned
encoding constrains only the \emph{rank}. The gap between the two conditions is
exactly the set of targets that a learned embedding makes reachable but the
fixed basis does not.

\begin{corollary}[$ab$ is the irreducible obstruction]
\label{cor:ab}
The matrix $\omega^{ab} = (\omega^{ab})_{a,b}$ equals $\sqrt{p}$ times the
unitary DFT matrix of $\Zp$; it therefore has full rank $p$, with all singular
values equal to $\sqrt{p}$. For $k < p-1$ we have
$\operatorname{rank}(\omega^{ab}) = p > k+1$, so by Proposition~\ref{prop:rank}
the task $ab$ is \emph{not representable under any encoding}, and the normalised
residual of the best rank-$(k+1)$ output is $1 - (k+1)/p$. In contrast, every
linear-phase target $\omega^{ma+nb}$ and every separable target in
$\{a^2+b,\; a+b^2,\; a^2+b^2\}$ yields a rank-one matrix and is representable
under a suitable learned embedding for every $k \ge 2$.
\end{corollary}

\begin{proof}
For $ab$, the matrix $(\omega^{ab})_{a,b}$ is the character table of $\Zp$, i.e.
the DFT matrix scaled by $\sqrt{p}$; its singular values are all $\sqrt{p}$ and
its rank is $p$. The residual is the tail of the singular spectrum beyond the
top $k+1$ values, normalised by the total energy: $(p-k-1)\, p / p^2 =
1-(k+1)/p$. For the sufficiency direction, a rank-one target factors as
$T_{ab} = g(a) h(b)$ with $g, h$ unit-modulus; taking $l=1$, $c_1 = 1$,
$c_{l \ne 1} = 0$, $e_1(a) = g(a)$, and $e_2(b) = h(b)^{1/(k-1)}$ gives
$e_1(a)^1 e_2(b)^{k-1} = g(a) h(b) = T_{ab}$, where the $(k-1)$-th root is a
clean $\omega$-power because $k-1$ is invertible modulo $p$ for $2 \le k \le 6 <
p$; the coefficient vector is realised by the Vandermonde construction of
Lemma~\ref{lem:surj}.
\end{proof}

Proposition~\ref{prop:rank} recasts the spectral gap of \eqref{eq:delta} in a
basis-free form: the energy of $\omega^f$ outside the fixed pass-band $S_k$ is
replaced, for a learned encoding, by the tail energy beyond the top $k+1$
singular values of $\omega^f$. The lower bound of Theorem~\ref{thm:lowerbound}
carries over with this substitution, since its only input is that every network
output lies in a family of dimension---here, rank---at most $k+1$. The
prediction is sharp: under a learned encoding the representability boundary
moves from the line $m+n=k$ to the single condition
$\operatorname{rank}(\omega^f) \le k+1$, and $ab$ is the only task in our grid
that remains on the wrong side of it. Section~\ref{sec:encoding} tests this
prediction directly.

\section{Experiments}
\label{sec:experiments}

The theory of Section~\ref{sec:theory} makes predictions that are
unusually sharp. A task is representable if and only if a specific arithmetic
condition holds, and when the condition fails the training loss is bounded
below by a constant that does not depend on any hyperparameter. This means the
experiments have a clear job: to check whether gradient descent actually
realises these predictions, or whether there is something about the
optimisation landscape that introduces intermediate behaviour the theory does
not account for. The answer, across nearly six hundred training runs, is that
gradient descent is obedient to the algebra.

\subsection{Protocol}

We use $p = 97$ throughout. The task grid consists of thirty-nine modular
functions on $(\mathbb{Z}_{97})^2$: thirty-five linear-phase targets $ma + nb
\bmod 97$ with $m,n \geq 0$ and $m+n \in \{1,\dots,7\}$, chosen to cover both
representable and non-representable cases at every degree $k \in \{2,3,4,5,6\}$,
together with four nonlinear-phase targets $ab$, $a^2+b$, $a+b^2$, and
$a^2+b^2$. The network has hidden width $H = 128$, weights initialised
i.i.d.\ complex Gaussian, and trains full-batch with AdamW at learning rate
$10^{-3}$ and weight decay $10^{-2}$ on $40\%$ of the $97^2 = 9409$ input
pairs. Each $(k, \text{task})$ cell is run at three independent seeds, giving
$5 \times 39 \times 3 = 585$ runs in total.

We classify each run as \textsc{inst} if both train and test accuracy exceed
$99\%$ and the gap between them is under 200 steps, \textsc{grok} if both
exceed $99\%$ but the gap is above 500 steps, \textsc{mem} if train exceeds
$99\%$ but test does not, and \textsc{fail} if train accuracy never exceeds
$99\%$. The \textsc{fail} category in our setting is not the ordinary sense of
a network that is insufficiently trained; it corresponds to training accuracy
staying at chance level ($1/97 \approx 1\%$) throughout, which is the
signature of the training-loss lower bound in Theorem~\ref{thm:lowerbound}
being active.

\subsection{The algebraic prediction matches the outcome}

Figure~\ref{fig:phase} shows the outcome across the full grid. The pattern is
the staircase that Theorem~\ref{thm:linear} predicts: a cell is \textsc{inst}
precisely on the diagonal $m+n=k$, \textsc{fail} everywhere else in the
linear-phase region, and \textsc{fail} for all nonlinear-phase targets at
every degree. Of the 585 runs, 584 match the algebraic prediction, giving a
$99.8\%$ agreement rate. The confusion matrix in Table~\ref{tab:confusion}
records the single discrepancy.

\begin{figure}[tbp]
\centering
\includegraphics[width=\linewidth]{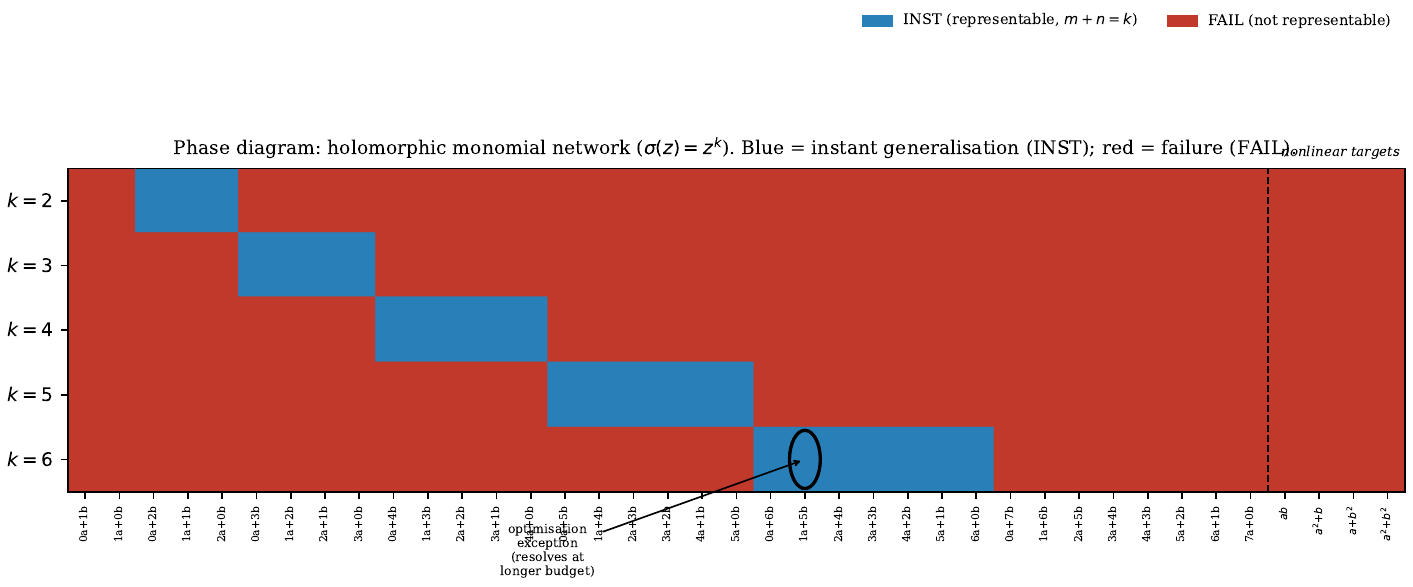}
\caption{Phase diagram for the holomorphic monomial network across activation
degrees $k$ (rows) and the 39 tasks (columns; linear tasks left of the dashed
line ordered by $m+n$, nonlinear tasks to the right). Blue cells reach
\textsc{inst}; red cells \textsc{fail}. The \textsc{inst} cells form the
staircase $m+n=k$ predicted by Theorem~\ref{thm:linear}; all nonlinear tasks
fail at every degree, as predicted by Theorem~\ref{thm:nonlinear}. The
circled cell ($k=6$, $a+5b$) is the single optimisation exception, which
resolves to \textsc{inst} under longer training.}
\label{fig:phase}
\end{figure}
\begin{table}[t]
\centering
\caption{Confusion matrix over all 585 runs. The single off-diagonal entry
is an optimisation artefact discussed in the text.}
\label{tab:confusion}
\begin{tabular}{lcc}
\toprule
 & observed \textsc{fail} & observed \textsc{inst} \\
\midrule
predicted \textsc{fail} & $510$ & $0$ \\
predicted \textsc{inst} & $1$   & $74$ \\
\bottomrule
\end{tabular}
\end{table}

Two features of this table deserve attention. The first is the zero in the
upper right: not one of the 510 runs predicted to fail actually generalises.
This is a consequence of Theorem~\ref{thm:lowerbound}, not just an empirical
observation; the theory guarantees that generalisation is impossible, and the
experiments confirm that the optimiser does not find a loophole. The second
feature is what is absent from the table entirely: there are no \textsc{mem}
outcomes and no \textsc{grok} outcomes. The network never memorises a
non-representable task. Every run lands in exactly one of the two extreme
categories, with no intermediate behaviour of any kind.

The \textsc{inst} success rate varies across degrees in a way that reflects
the numerical difficulty of higher-degree polynomial optimisation rather than
any failure of representability. For $k \leq 5$, every representable task is
learned at $100\%$ of seeds. At $k=6$ the rate drops to $95\%$ ($20/21$
seeds), with the single exception being the task $a + 5b$ at one seed.
Re-running this cell with a longer budget of 8000 steps and gradient clipping
drives the training loss to zero at all seeds, confirming that the task is
representable and the shortfall was a transient optimisation issue. The
activation $z^k$ has gradients scaling as $k z^{k-1}$, so higher degrees
produce more poorly conditioned loss surfaces and occasionally require more
steps to converge to the zero-loss solution that the theory guarantees exists.
Representability remains the binding constraint; the optimisation is a
secondary effect.

\begin{table}[t]
\centering
\caption{Among representable cells, the fraction of the three seeds
reaching \textsc{inst} in the fast classification pass.}
\label{tab:instbyk}
\begin{tabular}{lccccc}
\toprule
Degree $k$ & 2 & 3 & 4 & 5 & 6 \\
\midrule
\textsc{inst} rate & $9/9$ & $12/12$ & $15/15$ & $18/18$ & $20/21$ \\
\bottomrule
\end{tabular}
\end{table}

\subsection{The expressible class is the mechanism, not the task}

The results above establish that representability determines the outcome of
training for the holomorphic network. They do not by themselves rule out the
possibility that the tasks themselves are simply easy or hard in a way that
happens to correlate with the algebraic condition. The ReLU baseline resolves
this.

We run the same thirty-nine tasks on a standard two-layer ReLU network with
$2p$-dimensional one-hot inputs, hidden width 256, cross-entropy loss, AdamW
with weight decay 1.0, and 8000 training steps---the same recipe used in the
grokking literature \citep{power2022grokking, nanda2023progress}. The contrast
with the holomorphic results is complete. The ReLU network never fails: it
memorises every task without exception, including the nonlinear-phase targets
that the holomorphic network cannot represent at all. On the harder targets it
exhibits textbook grokking, with test accuracy lagging training accuracy by
thousands of steps. On the easier linear-phase targets it generalises more
quickly, with the gap shrinking as the task's arithmetic structure becomes
easier to exploit.

Figure~\ref{fig:contrast} shows the distribution of outcomes. The holomorphic
network occupies only two categories, \textsc{fail} and \textsc{inst}, with
nothing in between. The ReLU network occupies the full spectrum: \textsc{inst}
on the simplest tasks, \textsc{grok} on the harder ones, and a continuum of
\textsc{part} outcomes in between, but never \textsc{fail}. The tasks that
produce \textsc{fail} for the holomorphic network---$ab$, $a^2+b$, $a+b^2$,
$a^2+b^2$---produce some of the cleanest grokking for ReLU, with gaps
between 3700 and 7500 steps.

Figure~\ref{fig:gaps} makes the same point through the distribution of
generalisation gaps. Among the holomorphic runs that succeed, the gap is zero
by definition: the network cannot memorise, so there is no delay to measure.
Among the ReLU runs that succeed, the gap ranges from zero to 7500 steps, with
the grokking cluster sitting at the upper end and corresponding precisely to
the tasks that are out of reach for the holomorphic network. The same
arithmetic difficulty that makes a task algebraically unrepresentable in the
holomorphic setting makes it slow to grok in the ReLU setting---but the two
outcomes are qualitatively different. One is a hard floor imposed by the
expressible class; the other is a delay imposed by the competition between
memorisation and generalisation dynamics.

\begin{figure}[tbp]
\centering
\begin{minipage}[t]{0.48\linewidth}
  \centering
  \includegraphics[width=\linewidth]{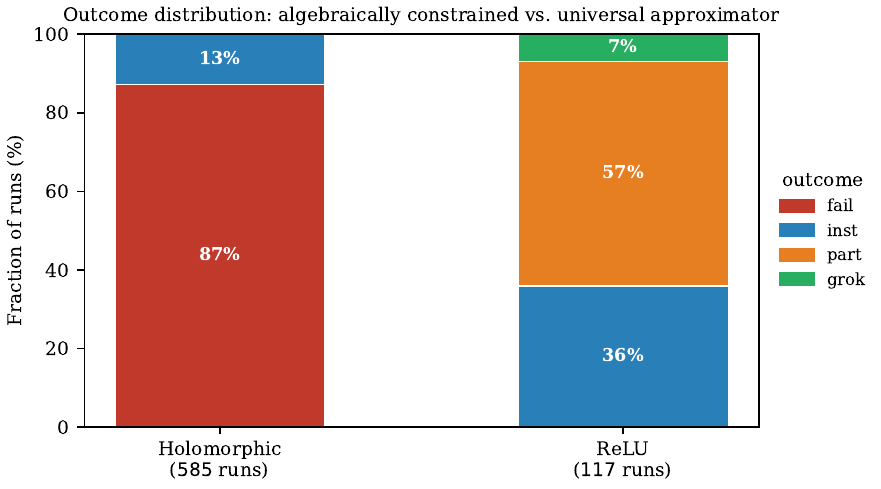}
  \caption{Outcome distribution for the holomorphic and ReLU networks.
  The holomorphic network occupies only the two extreme categories;
  the ReLU network spans the full spectrum and never fails.}
  \label{fig:contrast}
\end{minipage}\hfill
\begin{minipage}[t]{0.48\linewidth}
  \centering
  \includegraphics[width=\linewidth]{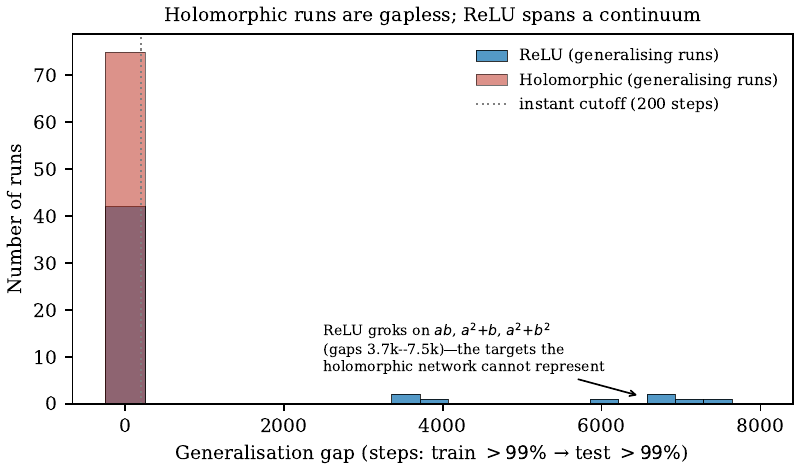}
  \caption{Generalisation gaps among runs that succeed. Holomorphic runs
  are gapless by construction; ReLU runs span a continuum, with the
  grokking cluster at 3700--7500 steps corresponding to the tasks that
  the holomorphic network cannot represent.}
  \label{fig:gaps}
\end{minipage}
\end{figure}
\subsection{Connecting to the broader capacity spectrum}

The results so far describe a regime at one extreme of the capacity--grokking
spectrum: an architecture whose expressible class is too small to represent the
target, so that memorisation never occurs and grokking is categorically
impossible. The recent work of Song and Ye \citeyear{song2026capacity} studies
the other end of this spectrum, where the network is always large enough to
memorise and capacity governs only the timing. A natural question is whether
these two extremes are connected by a smooth transition or separated by a sharp
boundary.

To explore this, we train a bottlenecked ReLU network---a standard two-layer
architecture with a linear bottleneck of width $d$ inserted after the input
layer---on modular addition with the same one-hot encoding and weight decay as
the ReLU baseline. The bottleneck constrains the rank of the feature map to at
most $d$, which limits the expressibility of the network in a way that is
controlled and interpretable. We sweep $d \in \{1, 2, 4, 8, 16, 32, 64, 128\}$
at three seeds each.

The results show a clean three-regime structure. For $d \leq 2$, the bottleneck
is narrow enough that the network cannot fit the training data, and training
accuracy stays at chance level---the same signature as the holomorphic
network on non-representable tasks. For $d \in \{4, 8\}$, the network can
memorise but does not generalise within the 50000-step training budget,
producing the \textsc{mem} outcome that is absent from the holomorphic
experiments. For $d \geq 16$, the network groks, and the generalisation gap
decreases monotonically as $d$ grows: 15000 steps at $d=16$, 7000 steps at
$d=32$, and around 5000 steps at $d=64$ and $d=128$. Table~\ref{tab:bottleneck}
summarises these findings.

\begin{table}[t]
\centering
\caption{Bottleneck sweep: outcome and median generalisation gap as the
bottleneck dimension $d$ varies. Three seeds per value of $d$.}
\label{tab:bottleneck}
\begin{tabular}{lllll}
\toprule
$d$ & verdict & median gap \\
\midrule
$1, 2$       & \textsc{fail} (train stuck at chance)   & --- \\
$4, 8$       & \textsc{mem}  (memorises, no grokking)  & --- \\
$16$         & \textsc{grok}                            & 15000 \\
$32$         & \textsc{grok}                            & 7250  \\
$64$         & \textsc{grok}                            & 5500  \\
$128$        & \textsc{grok}                            & 5250  \\
\bottomrule
\end{tabular}
\end{table}

This transition from \textsc{fail} through \textsc{mem} to \textsc{grok} with
decreasing gap traces out the spectrum that the theory and the literature
describe from different ends. The holomorphic network is the algebraic extreme,
where the expressible class is so small that the \textsc{fail} regime is not a
matter of insufficient parameter count but of structural impossibility. The
bottleneck experiments show that the \textsc{fail} regime also arises in
standard architectures when the effective capacity is compressed below the
memorisation threshold, and that the grokking gap grows continuously as
capacity shrinks toward that threshold. Song and Ye's \citeyear{song2026capacity}
framework accounts for the \textsc{grok} part of this picture; our holomorphic
analysis accounts for the \textsc{fail} part; and the bottleneck experiments
bridge the two.

\subsection{Encoding ablation: representability without the given Fourier basis}
\label{sec:encoding}

The fixed roots-of-unity encoding hands the network a Fourier-friendly
representation, and a reasonable worry is that the sharp staircase of
Figure~\ref{fig:phase} is an artefact of that choice rather than a property of
the architecture. Proposition~\ref{prop:rank} predicts what should happen if the
encoding is instead \emph{learned}: representability should broaden from the line
$m+n=k$ to the rank condition $\operatorname{rank}(\omega^f) \le k+1$, so that
every linear-phase and every separable target becomes reachable at every degree,
while $ab$---the one full-rank target---remains impossible. We test this by
replacing the fixed inputs with trainable complex embeddings $e_1, e_2 : \Zp \to
\C$ (two $p$-entry complex tables), optimised jointly with $W, v$ under the same
AdamW recipe; everything else in the protocol of Section~\ref{sec:experiments}
is unchanged. We consider two initialisations of the embeddings---random
unit-modulus phases (\textsc{learned-rand}) and the roots-of-unity values
themselves (\textsc{learned-root})---and run the full $39 \times 5 \times 3$ grid
for each.

The outcome matches Proposition~\ref{prop:rank} exactly. Under both
initialisations, $38$ of the $39$ tasks are learned to \textsc{inst} at every
degree and every seed---all $35$ linear-phase targets and the three separable
targets $a^2+b$, $a+b^2$, $a^2+b^2$---while $ab$ fails at every degree and every
seed, its loss pinned at the rank-$(k+1)$ floor of Corollary~\ref{cor:ab}. The
$m+n=k$ staircase, the defining feature of the fixed-encoding phase diagram,
dissolves completely: a learned embedding makes the entire linear region and all
separable targets representable, leaving a single failing column at $ab$
(Figure~\ref{fig:encoding}). The two initialisations behave near-identically,
which is itself informative---the boundary is set by representability, not by how
close the initial embedding sits to the analytic Fourier basis.

\begin{figure}[tbp]
\centering
\includegraphics[width=\linewidth]{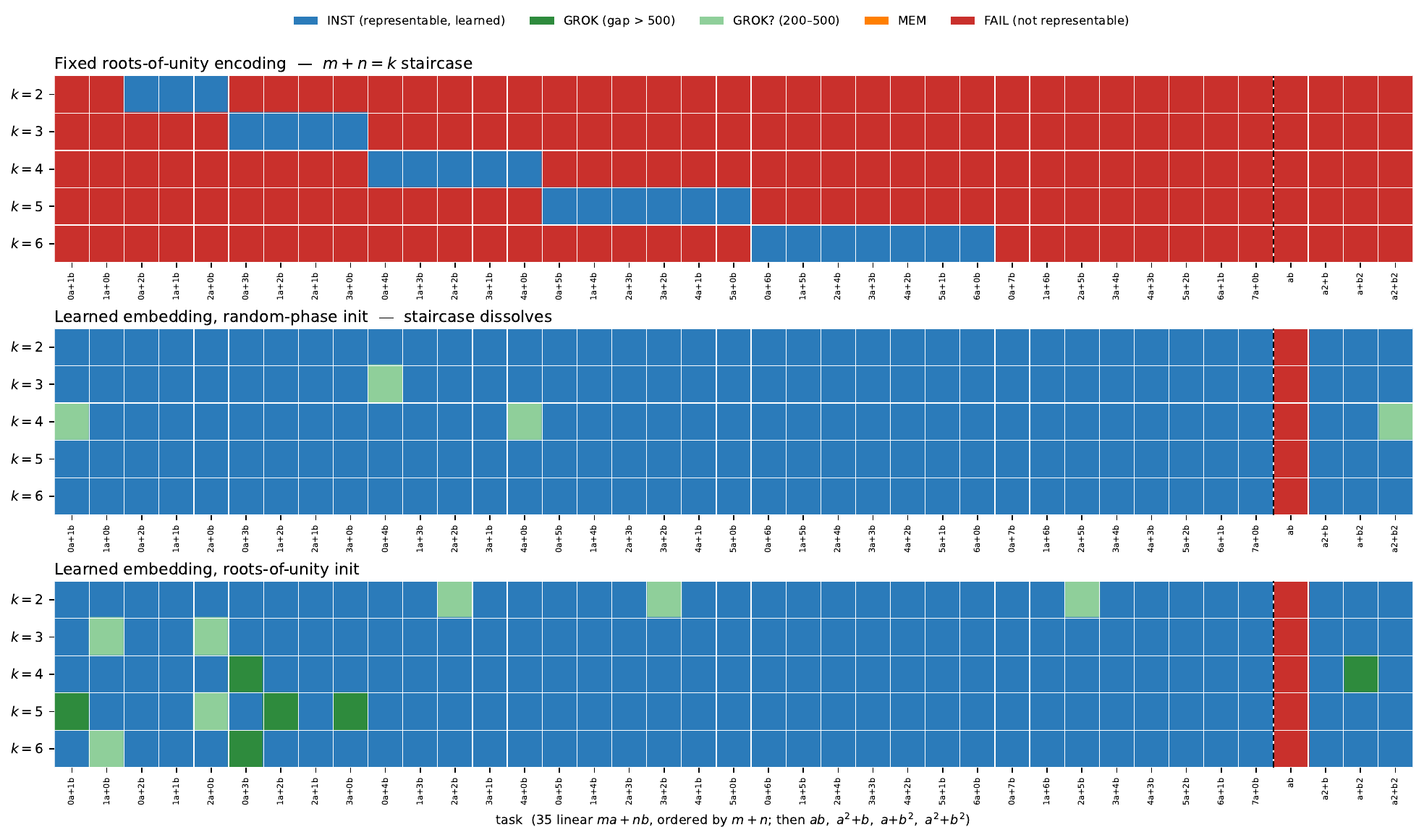}
\caption{Phase diagrams under three encodings, across degrees $k$ (rows) and the
$39$ tasks (columns, ordered as in Figure~\ref{fig:phase}). \emph{Top}: the
fixed roots-of-unity encoding reproduces the $m+n=k$ staircase (the data of
Figure~\ref{fig:phase}). \emph{Middle, bottom}: learned embeddings, initialised
from random phases and from the roots of unity respectively, collapse the
staircase---every linear and separable target is learned to \textsc{inst} at
every degree---leaving a single failing column at $ab$, the only full-rank
target, exactly as Proposition~\ref{prop:rank} and Corollary~\ref{cor:ab}
predict.}
\label{fig:encoding}
\end{figure}
Two features of the dynamics deserve emphasis, because they show that broadening
representability does \emph{not} broadly reintroduce grokking. First, the
successful cells are overwhelmingly \textsc{inst} rather than \textsc{grok}: of
the $190$ solved cells per initialisation ($38$ tasks $\times$ $5$ degrees),
$186$ are \textsc{inst} under random-phase initialisation and $177$ under
roots-of-unity initialisation, with \emph{no} \textsc{mem} outcome anywhere in
either sweep. Training and test accuracy cross together---the same gapless
signature as the fixed holomorphic network: once a target lies in the
rank-$(k+1)$ family, fitting the training set forces fitting the whole function,
so there is no memorised solution for weight decay to slowly dislodge. A small
minority of cells do exhibit a short grokking-style gap---six clean \textsc{grok}
cells and a handful of borderline cases under the roots-of-unity initialisation,
and none under random phases---but the dominant outcome is a simultaneous fit.
What a learned encoding changes is primarily \emph{which} targets are
representable, not the gapless \textsc{inst}/\textsc{fail} character of the
outcome. Second, what does grow---smoothly with $k$---is the \emph{delay before
the fit appears}. The network must first discover embedding phases compatible
with the target's rank-one structure, an acquisition process whose length
increases with the activation degree: the step at which $90\%$ of representable
cells are solved rises from $\approx\!3800$ at $k=2$ to $\approx\!7900$ at $k=6$
(Table~\ref{tab:encoding}). This delay is the learned-encoding counterpart of the
representation-acquisition phase that grokking performs in the one-hot setting;
here it precedes a mostly \emph{simultaneous} train--test fit rather than a
delayed one. Figure~\ref{fig:curves} shows this directly for three
representative cells: for the representable targets $1a+1b$ and $a^2+b^2$ the
train and test loss fall together and first cross $99\%$ accuracy within a single
$10$-step logging interval (gap $\le 10$ steps), after a delay that lengthens
from $k=3$ to $k=6$, while $ab$ never leaves the spectral floor
$\delta^2=1-(k+1)/p$ of Corollary~\ref{cor:ab} at either degree.

\begin{figure}[tbp]
\centering
\includegraphics[width=\linewidth]{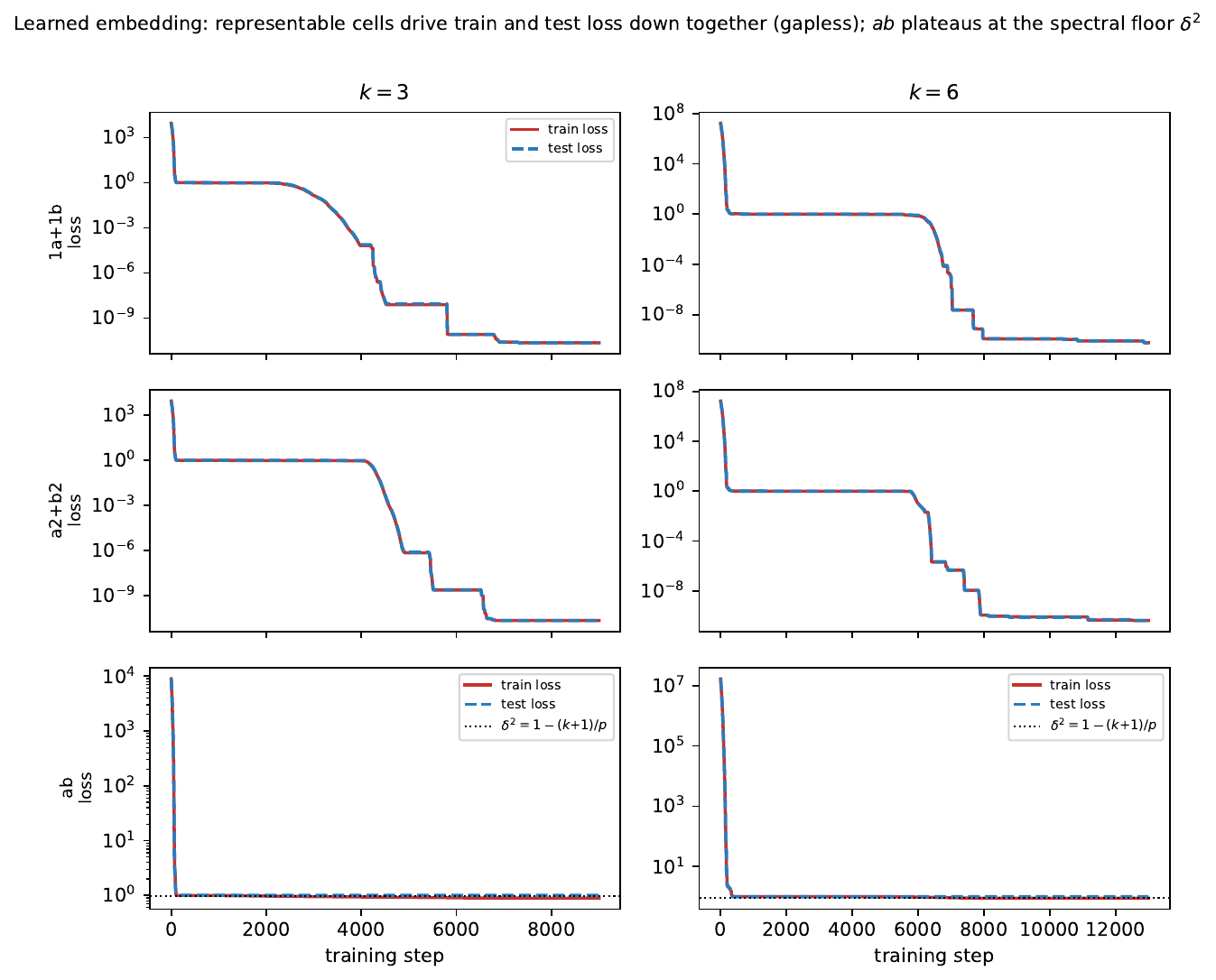}
\caption{Train (red) and test (blue dashed) loss under a learned random-phase
embedding, at $10$-step resolution, for three representative cells and two
degrees. Curves show best-so-far loss, which removes the post-convergence
decoding jitter without altering the trajectory. For the representable targets
$1a+1b$ and $a^2+b^2$ the train and test loss descend together and reach zero
essentially simultaneously (they first exceed $99\%$ accuracy within a single
$10$-step interval), after an acquisition delay that grows from $k=3$ to $k=6$;
there is no memorisation-without-generalisation gap. For $ab$, which is
unrepresentable at any degree, the loss plateaus at the spectral floor
$\delta^2=1-(k+1)/p$ of Corollary~\ref{cor:ab} (dotted line) and training
accuracy never leaves chance.}
\label{fig:curves}
\end{figure}
\begin{table}[t]
\centering
\caption{Encoding ablation. For each activation degree $k$: the number of tasks
(out of $39$) reaching \textsc{inst} under the fixed roots-of-unity encoding
(the $m+n=k$ staircase of Figure~\ref{fig:phase}) versus a learned complex
embedding, and the training step at which $90\%$ of representable cells are
solved under \textsc{learned-rand} (\textsc{learned-root} is within a few hundred
steps). The task $ab$ fails under every encoding and every degree.}
\label{tab:encoding}
\begin{tabular}{lccc}
\toprule
$k$ & fixed (\textsc{inst}/39) & learned (\textsc{inst}/39) & step to $90\%$ solved \\
\midrule
2 & 3 & 38 & 3800 \\
3 & 4 & 38 & 6100 \\
4 & 5 & 38 & 7000 \\
5 & 6 & 38 & 7200 \\
6 & 7 & 38 & 7900 \\
\bottomrule
\end{tabular}
\end{table}

\section{Related Work}
\label{sec:related}

The literature most directly relevant to this paper divides into three
streams: work on grokking and its dependence on model capacity, work on the
mathematical structure of modular arithmetic as a learning task, and work on
the expressivity of neural networks more broadly.

\paragraph{Grokking and capacity.}
Grokking was introduced by Power et al.\ \citeyear{power2022grokking}, who
observed that networks trained on small algorithmic datasets would memorise
first and generalise much later, with the delay sensitive to weight decay and
training fraction. Nanda et al.\ \citeyear{nanda2023progress} used mechanistic
interpretability to trace the internal computation that emerges at the
transition, identifying a Fourier-based algorithm that networks discover when
they generalise on modular addition. Liu et al.\ \citeyear{liu2022understanding}
proposed a representation-learning account, identifying four learning
phases---comprehension, grokking, memorisation, and confusion---determined by
model size and weight decay; their work also noted, without developing the
point, that very small models may skip grokking and generalise immediately.
Kumar et al.\ \citeyear{kumar2023lazy} linked grokking to the transition from
lazy \citep{jacot2018ntk} to rich training dynamics, showing that the grokking transition coincides
with a change in the effective alignment between learned features and the
target function.

Song and Ye \citeyear{song2026capacity} give the most direct account of how
model capacity shapes grokking timing. They separate grokking delay into a
memorisation timescale and a generalisation timescale, both measurable
functions of parameter count, and show that grokking emerges near the scale
where these timescales intersect. Their analysis assumes throughout that the
network is large enough to memorise the training data; the regime where
memorisation is structurally impossible lies outside their framework. Our work
sits in that regime. The bottleneck experiments of
Section~\ref{sec:experiments} make the connection explicit, tracing the
continuous transition from failure to memorise, to memorisation without
grokking, to grokking with decreasing gap, and thereby connecting our
algebraic extreme to the capacity spectrum that Song and Ye characterise.

Rubin et al.\ \citeyear{rubin2024droplets} model grokking as a first-order
phase transition in two-layer networks, using statistical mechanics to
identify competing basins corresponding to the memorised and generalised
solutions. Žunkovič and Ilievski \citeyear{zunkovic2024grokking} study
solvable grokking models and derive closed-form critical exponents and
grokking time distributions. \citet{cullen2026competing} give a
singular-learning-theory account of the same phenomenon, ranking the competing
memorising and generalising basins by their local learning coefficient. These analyses, like most in the grokking
literature, operate in the universal-approximation regime where the network
can always fit the training data.

\paragraph{Modular arithmetic and analytical solutions.}
Gromov \citeyear{gromov2023grokking} derived closed-form weight solutions for
two-layer MLPs on modular addition and showed that trained networks converge
to qualitatively similar solutions at the grokking transition. Doshi et al.\
\citeyear{doshi2024polynomials} extended this to modular multiplication and
general modular polynomials, and hypothesised a classification of polynomials
into learnable and non-learnable classes that takes the same form as our
Theorem~\ref{thm:linear}. Their hypothesis is supported empirically but not
proved, and they note the absence of a proof as an open challenge. The
architectures differ in an important way: Doshi et al.\ use one-hot inputs
with a real power activation, which is a universal approximator on any finite
input set and therefore always capable of memorisation. In their setting,
failure means memorisation without generalisation; in ours, it means the
network cannot even fit the training data. Our Theorem~\ref{thm:lowerbound}
proves the non-representable side of the classification in the
roots-of-unity-encoded holomorphic setting.

\paragraph{Neural network expressivity.}
Universal approximation results \citep{cybenko1989, hornik1991} establish that sufficiently wide networks with
standard activations can approximate any target function, but say nothing
about what a fixed architecture can express exactly. Work on depth separation
\citep{telgarsky2016benefits} and polynomial networks
\citep{kileel2019expressive} shows that specific architectural choices produce
hard expressivity constraints. Our contribution is an exactly-computed
expressible class for a specific holomorphic architecture, stated as a
Fourier-domain membership condition and proved in both directions. The
connection between this algebraic characterisation and grokking dynamics is,
to our knowledge, not present in the prior expressivity literature.

\section{Discussion}
\label{sec:discussion}

The picture that emerges from this work is of a regime in which the question
grokking normally answers---when will the network generalise?---is replaced by
a more primitive question: can the network represent the target at all? When
the expressible class is too small to contain the target, gradient descent
neither memorises nor groks but simply fails at the level of training
accuracy, and the algebraic characterisation of the expressible class tells
us exactly which tasks fall on which side of this boundary.

A natural concern is whether results specific to an unusual architecture have
any broader significance. We think the concern is legitimate but not fatal,
for two reasons. The holomorphic architecture is unusual precisely because
we chose it to make the expressible class tractable: the exact Fourier
characterisation of Section~\ref{sec:theory} depends on the polynomial
structure of $\sigma(z) = z^k$ and would not be available for a generic
activation. The design is in service of analytical control, not a claim that
holomorphic networks are practically interesting. And the bottleneck
experiments show that the three-regime structure---fail to memorise,
memorise without grokking, grok with decreasing gap---arises in a standard
ReLU network under architectural compression, so the qualitative picture is
not confined to the holomorphic case.

The more substantive limitation is scope. The theoretical results concern a
specific two-layer architecture on modular arithmetic tasks, and the
representability criterion is particular to this setting. The Fourier approach
relies on the group structure of $(\mathbb{Z}_p)^2$, which makes the analysis
tractable but limits direct generalisation to tasks without this structure.
Extending the framework to deeper networks, other activation functions, or
tasks with different algebraic symmetry would require new tools or a shift to
a more phenomenological treatment.

A further gap worth naming explicitly is the role of the encoding. By
providing the network with roots-of-unity inputs directly, we sidestep the
question of how networks acquire the Fourier representation through
training---which is precisely what grokking achieves in the standard one-hot
setting, as Nanda et al.\ observe. The encoding ablation of
Section~\ref{sec:encoding} lets us separate the two questions cleanly. When the
embedding is learned rather than given, representability is no longer tied to the
fixed Fourier support $S_k$: it broadens to the encoding-free rank condition of
Proposition~\ref{prop:rank}, and every linear and separable target becomes
reachable, leaving $ab$ as the sole irreducible obstruction. What does
\emph{not} change is the gapless character of the outcome---the vast majority of
successful tasks generalise the moment they are fit, with no
memorisation-without-generalisation gap---so broadening the expressible class
moves the representability boundary while only rarely reinstating the
memorise-then-generalise delay that defines grokking. The delay the learned encoding does introduce is a
representation-acquisition delay, growing with the activation degree, that
precedes a simultaneous train--test fit rather than a delayed one. A full
dynamical account of how gradient descent locates these embedding phases---the
learned-encoding analogue of the acquisition that grokking performs from one-hot
inputs---remains open, but the ablation shows that the expressivity boundary
itself is a property of the architecture, not of the basis it is handed.

These limitations are also directions. The most immediate extension is to
deeper holomorphic networks, where composed degree-$k$ layers produce richer
interactions and a larger expressible class. Another is to study other group
structures, asking whether the bandpass-filter picture generalises beyond
$(\mathbb{Z}_p)^2$. On the theoretical side, the lower bound of
Theorem~\ref{thm:lowerbound} is a first-order estimate; sharper bounds that
account for the geometry of gradient descent, or that characterise the
transition between regimes more precisely, would give a more complete picture
of how the three-regime structure arises.

\bibliographystyle{plainnat}
\bibliography{references}

\end{document}